%% file: Improving-offline-evaluation-of-contextual-bandit-algorithms-via-bootstrapping-techniques.tex
\documentclass{article}

\usepackage{times}
\usepackage{graphicx} 
\usepackage{subfigure} 

\usepackage{natbib}

\usepackage{algorithm}

\usepackage{hyperref}

\usepackage[accepted]{icml2014}

\usepackage{ulem} 

\icmltitlerunning{Bootstrapped evaluation of contextual bandit algorithms}

\usepackage{amsfonts,amsmath}

\usepackage[utf8]{inputenc} 
\newcommand{\D}{{\cal D}}
\newcommand{\proba}{\mathbb{P}r}
\usepackage{array}

\usepackage[usenames,dvipsnames]{xcolor}

\usepackage{graphics}
\usepackage{xspace}

\newcommand{\knownRatings}{{\cal S}}

\graphicspath{{figures/}}
\usepackage{tikz}
\usetikzlibrary{arrows,snakes,shapes}

\input{raccourcis.tex}

\newcommand{\SoA}{{\cal A}}

\newcommand{\keepForLongVersionIfAny}[1]{}

\usepackage{amsthm}
\newcommand\isdef{\stackrel{\text{\tiny def}}{=}}

\usepackage{amsthm}
\newtheorem{theorem}{Theorem}

\usepackage[usenames,dvipsnames]{xcolor}

\newcommand\algo{A}
\newcommand\reward[1]{\vec{r_t}[#1]}

\begin{document}

\twocolumn[
\icmltitle{Improving offline evaluation of contextual bandit
  algorithms \\ via bootstrapping techniques}
\icmlauthor{Olivier Nicol}{oli.nicol@gmail.com}
\icmlauthor{Jérémie Mary}{jeremie.mary@inria.fr}
\icmlauthor{Philippe Preux}{philippe.preux@univ-lille3.fr}
\icmladdress{University of Lille / LIFL (CNRS) \& INRIA Lille Nord Europe, 59650 Villeneuve d'Ascq, France}

\icmlkeywords{off-line evaluation, recommender systems, contextual bandit, bootstrap}

\vskip 0.3in
]

\begin{abstract} In many recommendation applications such as news
recommendation, the items that can be recommended come and go at a
very fast pace.
This is a challenge for recommender systems (RS) to face this setting.
  Online learning algorithms seem to be the most straight
forward solution. The contextual bandit framework was introduced for
that very purpose. In general the evaluation of a RS is a critical
issue. Live evaluation is often avoided due to the potential loss of
revenue, hence the need for offline evaluation methods. Two options
are available. Model based methods are biased by nature and are thus
difficult to trust when used alone. Data driven methods are therefore
what we consider here. Evaluating online learning algorithms with past
data is not simple but some methods exist in the
literature. Nonetheless their accuracy is not satisfactory mainly due
to their mechanism of data rejection that only allow the exploitation
of a small fraction of the data. We precisely address this issue in
this paper. After highlighting the limitations of the previous
methods, we present a new method, based on bootstrapping
techniques. This new method comes with two important improvements: it
is much more accurate and it provides a measure of quality of its
estimation. The latter is a highly desirable property in order to
minimize the risks entailed by putting online a RS for the first
time. We provide both theoretical and experimental proofs of its
superiority compared to state-of-the-art methods, as well as an
analysis of the convergence of the measure of quality.
\end{abstract}

\section{Introduction}

Under various forms, and under various names, recommendation has
become a very common activity over the web. One can think of movie
recommendation (Netflix), e-commerce (Amazon), online advertising
(everywhere), news recommendation (Digg), personalized radio stations
(Pandora) or even job recommendation (LinkedIn)... All these
applications have their own characteristics. Yet the common key idea
is to take advantage of some information we may have on a user
(profile, demographics, time of the day \textit{etc.}) in order to
identify the most attractive content to serve him/her in a given
context. Note that an element of content is generally referred to as
an item. To perform recommendation, a piece of software called a
recommender system (RS) can use past user/item interactions 
such as clicks or ratings. In particular, the typical approach to
recommendation is to train a predictor of ratings and/or clicks of
users on items on past data and use the resulting predictions to make
personalized recommendations. This approach is based on the implicit
assumption that past behavior can be used to predict future behavior.

In this paper we consider a recommendation applications in which the
aforementioned assumption is not reasonable. We refer to this setting
as dynamic recommendation. In dynamic recommendation, only a few tens
of different items are available for recommendation at any given
moment. These items have a limited lifetime and are continuously
replaced by new ones, with different characteristics. We also consider
that the tastes of users change, sometimes dramatically due to
external parameters that we do not control. Many examples of dynamic
recommendation exist on the web. The most popular one is news
recommendation that can be found on specialized websites such as Digg
or on general web portals (Yahoo!) and websites of various media
(newspapers, TV channels...). Other examples can be mentioned such as
private auctions in which the user can buy a limited set of items that
changes everyday. Another example is a RS that can only recommend the
K most recent items (this may apply to movies, videos, songs...). This
problem has begun to be addressed recently with online learning
solutions, by considering the contextual bandit framework. Nonetheless
this is not the case in most of the recommendation literature. In all
the textbooks, dynamic recommendation is handled with content-based
recommendation. The idea is to consider an item as a set of features
and to try to predict the taste of a user with regards to these
features, using an offline predictor as before. Yet we argue that
although this can be a good idea in some special cases, this is not
the way to go in general:
\begin{enumerate}
  \item It requires a continuous labeling effort of new items.
  \item We are limited by what the expert labels: things can be hard
    to label such as the appeal for a picture, the quality of a textual
    summary, \textit{etc.}
  \item Tastes are not static: the appeal of a user to some kind of
    news can be greatly impacted by the political context. Similarly
    the appeal towards clothing can be impacted by fashion, movie
    stars \textit{etc.}
\end{enumerate}

For such systems, the best way to compare the performance of two
algorithms is to perform A/B testing on a subset of the web audience
\cite{ABtesting}. Yet there is almost no e-commerce website that would
let a new RS go live just for testing, even on a small portion of the
audience for fear of hurting the user experience and loosing
money. The entailed engineering effort can also be
discouraging. Therefore being able to evaluate offline a RS is
crucial. In classical recommendation, the measures of prediction
accuracy and other metrics that can be computed on past data are well
accepted and trusted in the community. Nevertheless for the reasons we
gave above, they are irrelevant for online learning algorithms
designed for dynamic recommendation. This paper is about the offline
evaluation of such algorithms. Some fairly simple replay methodologies
do exist in the literature. Nonetheless they have a well known, yet
very little studied drawback which is that they use a very small
portion of the data at hand. One may argue that the tremendous amount
of data available with web applications makes this a marginal
problem. Yet in this paper we will exhibit that this is a \emph{major}
issue that generates a huge bias when evaluating online learning
algorithms for dynamic recommendation. Furthermore we will explain
that acquiring more data \emph{does not} solve the problem. Then we
will propose a solution to this issue, that builds on the previously
introduced methods and on different elements of bootstrapping
theory. This solution is backed by a theoretical analysis as well as
an empirical study. They both clearly exhibit the improvement in terms
of evaluation accuracy brought by our new method. Furthermore the use
of bootstrapping allows us to estimate the distribution of our
estimation and therefore to get a sense of its quality. This is a
highly desirable property, especially considering that such an
evaluation method is designed in order to decide whether we should
risk putting a new algorithm online. The fast theoretical convergence
of this estimation is also proved in our analysis. Note that the
experiments are run on synthetic data, for reasons that we will detail
and also on a large publicly available dataset.

\section{Background on bandits and notations}

We motivated the need for online learning solutions in order to deal
with dynamic recommendation. A natural
way to model this situation is as a reinforcement learning problem
\cite{sutton1998reinforcement}, and more precisely using the
contextual bandit framework \cite{contextualMAB:ai-stats2010} that was
introduced for the very purpose of news recommendation.

\subsection{Contextual Bandits}
\label{cbandits}

The bandit problem is also known in the literature as the multi-arm
bandit problem and other variations. This problem can be traced back
to Robbins and Munro in 1952 \cite{Robbins:1952} and even Thompson in
1933 \cite{thompson1933}. There are many variations in the definition
of the problem; the contextual bandit framework as it is defined in
\cite{DBLP:conf/nips/LangfordZ07} is as follows:

Let $\cal X$ be an arbitrary input space and
$\SoA=\left\{1..K\right\}$ be a set of $K$ actions. Let $\D$ be
a distribution over tuples $(x, \vec{r})$ with $x \in {\cal X}$ and
$\vec{r} \in \left\{0,1\right\}^K$ a vector of rewards: in the $(x,
\vec{r})$ pair, the $j^{\mbox{\scriptsize th}}$ component of $\vec{r}$
is the reward associated to action $a_j$, performed in the context
$x$. In recommendation, the context is the information we have on the
user (session, demographics, profile \textit{etc.}) and an action is
the recommendation of an item.

A contextual bandit game is an iterated game in which at each round
$t$:
\begin{itemize}
  \item $(x_t,\vec{r_t})$ is drawn from $\D$.
  \item $x_t$ is provided to the player.
  \item The player chooses an action $a_t \in \SoA$ based on
    $x_t$ and on the knowledge it gathered from the previous rounds of
    the game.
  \item The reward $\reward{a_t}$ is revealed to the player whose
    score is updated.
  \item The player updates his knowledge based on this new experience.
\end{itemize}

It is important to note the partial observability of this game:
the reward is known only for the action performed by the player. This
is what makes offline evaluation a complicated problem. A typical goal
is to maximize the sum of reward obtained after $T$ steps. To
succeed a player has to learn about $\D$ and try to exploit that
information. Therefore at time $t$, a player faces a classical
exploration/exploitation dilemma: either perform an action he is
uncertain about in order to improve his model of $\D$ (explore), either perform
the action he believe to be optimal, although it may not be (exploit). 

A simpler variant of this problem in which no contextual information
is given, called the multi armed bandit problem (MAB) was extensively
studied in the literature. One can for instance mention the UCB
algorithm \cite{Auer02finite-timeanalysis} that optimistically deals with the dilemma by
performing the action with higher upper confidence bound on the
estimated reward expectation. The contextual bandit problem is less
studied due to the additional complexity and additional assumptions
entailed by the context space. The most popular algorithm is without
doubt LinUCB \cite{LinUCB}, although a few others exist such as
epoch-greedy \cite{DBLP:conf/nips/LangfordZ07}. LinUCB is basically an
extension of the classical UCB that uses the contexts under the
assumption of normality and that ${\cal X}=\R^d$. The reward expectation of an
action is estimated via a linear regression on the context vectors and
the confidence bound is computed using the dispersion matrix of the
context vectors. These two state-of-the-art algorithms are the ones we
will evaluate when running experiments.

\subsection{Evaluation}

We define a contextual algorithm $\algo$ as taking as input an
ordered list of $(x,a,r)$ triplets (history) and outputting a policy $\pi$. A policy
$\pi$ maps $\cal{X}$ to $\SoA$, that is chooses an action given a
context. Note that we are also interested in evaluating policies. In
our setting which is the most popular one, an algorithm is said
optimal when maximizing the expectation of the sum of rewards after $T$
steps: 
$$
  G_\algo(T,\D) \isdef \E_\D \sum_{t=1}^T \reward{\algo(x_t)} .
$$
For convenience, we define the per-trial payoff as the average click
rate after $T$ steps:
$$
g_\algo(T,\D) \isdef \frac{G_\algo(T)}{T}
$$
Note that for a static algorithm (\textit{ie.} that always outputs the
same policy $\pi$), we have that:
$$
\forall T, g_\algo(T,\D) = g_\algo(1,\D) \isdef g_\algo(\D).
$$
Note that from this point on, we will simplify the notations by
systematically dropping $\D$. \textbf{$g_\algo(T)$ is thus the quantity
  we wish to estimate as the measure of performance of a bandit algorithm.}

In order to minimize the risks entailed by playing live a new
algorithm, we are also interested in the quality of the
estimation. Bootstrapping will enable us to estimate it. To do so we need
additional notations. $CTR_\algo(T)$ denotes the distribution of the
per-trial payoff of $\algo$ after $T$ steps (so $g_\algo(T)$ is its
mean). Besides estimating $g_\algo(T)$, our second goal is the
computation of an estimator quality assessment
$\xi\left(CTR_\algo(T)\right)$. Note that typically, $\xi$ can be a
quantile, a standard error, a bias or what we will consider here for
simplicity: a confidence region around the mean of $CTR_\algo(T)$ (aka $g_A$).  

\section{The time acceleration issue with replay methodologies}

This section describes the replay methodology, that we call
\textit{replay} and that was introduced by
\cite{Langford_ExploScav_08} and analyzed for the setting we consider
by \cite{LiCLW11}. This section also highlights the method's
limitations that we overcome in this paper and is crucial to understand the
significance of our contribution. 

First of all, as \cite{LiCLW11}, we assume that we have a dataset $S$
that was acquired online using an random uniform allocation of the actions for $T$ steps. This data collection phase can be referred as \textit{exploration policy} and is our unique information on $\D$. 
  This  random decision making implies that any point in ${\cal X} \times \SoA$ has a non null probability to belong to $S$; this allows the evaluation of any policy. In a nutshell, the replay
methodology on such a dataset works as follows: for each record
$(x_t,a_t,r_t)$ in $S$, the algorithm $\algo$ is asked to choose an
action given $x_t$. If this action is $a_t$, $r_t$ is revealed to
$\algo$ and taken into account to evaluate its performance. If the
action is different, the record is discarded. This method is proved to
be unbiased in some cases \cite{LiCLW11}. Note that the fact it needs
the data to be acquired uniformly at random is quite restrictive. This
problem is well studied and \textit{replay} can be extended to allow
the use of data acquired by a different but known logging policy at a
cost of increased variance \cite{Langford_ExploScav_08}. Some work has
been done to reduce this variance and even allow the use of data for
which the logging policy is unknown
\cite{DBLP:journals/corr/abs-1103-4601,  DBLP:conf/nips/StrehlLLK10}. Note also that if the evaluated bandit
algorithm is close from the logging policy, we may even further reduce
the variance \cite{tr-bottou-2012}. Finally there exist ways to adapt this method
to the case where a list of items can be recommended \cite{Langford_ExploScav_08}. Although we do not take into
account these considerations and keep the simplest assumption in this
paper for clarity, our method is based on the same ideas as
\textit{replay} and could therefore be extended similarly as what is
presented in the works we just cited.

Another issue with \textit{replay} is well-known but not studied at
all up to our knowledge. In average, only $\frac{T}{K}$ records are
used. Therefore \textit{replay} outputs an estimate
$\hat{g_\algo}\left(\frac{T}{K}\right)$ which follows the distribution
$CTR_\algo\left(\frac{T}{K}\right)$ of mean
$g_\algo\left(\frac{T}{K}\right)$. It is important to have in mind
that $g_\algo\left(\frac{T}{K}\right) = g_\algo\left(T\right)$
\emph{if and only if} $T=+\infty$ or $\algo$ is a static policy. See
figure \ref{fig:biased_replay} for a visual example of this problem.
Note that in \emph{any situation} except $K=1$,
$CTR_\algo\left(\frac{T}{K}\right) \neq CTR_\algo\left(T\right)$, and
the same thing goes for the confidence region $\xi$.

One may argue that when evaluating a RS, plenty of data is
available and therefore that $T$ is almost infinite. Consequently one may
also consider \textit{replay} to be almost unbiased. This is true with
the classical contextual bandit framework considered in the literature. With dynamic
recommendation, the main application for this method, this could not
be more wrong. Indeed, we argued that in this context, everything
changes, especially the available items. For instance, in news
recommendation a news remains active from a few hours to two days
and its CTR systematically decreases over time
\cite{Agarwal:2009:SME:1526709.1526713}. Moreover we also mentioned
reasons to believe that the user tastes may change as well. Therefore
when evaluating a contextual bandit algorithm, we want to evaluate its
behavior against a very dynamic environment and in particular its
ability to adapt to it. The use of replay in such a context is often
justified by the fact that the environment can be considered static
for small periods of time. This is not necessary but makes the
understanding of our point easier. When an algorithm faces a "static" region of
a dataset of size $T_i$, when being replayed, it only has
$\frac{T_i}{K}$ instead of $T_i$ steps to learn and exploit that
knowledge. It is impossible to solve this problem by considering more data
since new data would concern the next region, where different news
with different CTR are available for recommendation, and potentially
users with different tastes. In fact whatever
assumptions we use to characterize how things evolve, using
replay is equivalent to playing an algorithm with time going $K$ times
faster than in reality. This generates a huge bias. Note that it is
most likely because of time acceleration that a non-contextual
algorithm which looks a lot like UCB won a challenge evaluated by
\textit{replay} on the Yahoo! R6B dataset \cite{yahooToday} (news
recommendation). See chapter 4 of Nicol \cite{nicolEval14} for more
details. 

As a conclusion, we consider in this work a classical contextual
bandit framework with a fixed number of steps $T$. We assume that no
more than $T$ records can be acquired. 
Yet it is clear that if we manage to deal with this
problem without adding data, we would also be able to deal as well with the
problem of evaluating dynamic recommendation for which using more
data may not be possible.

\begin{figure}[ht]
      \centering
      \includegraphics[width=0.4\textwidth]{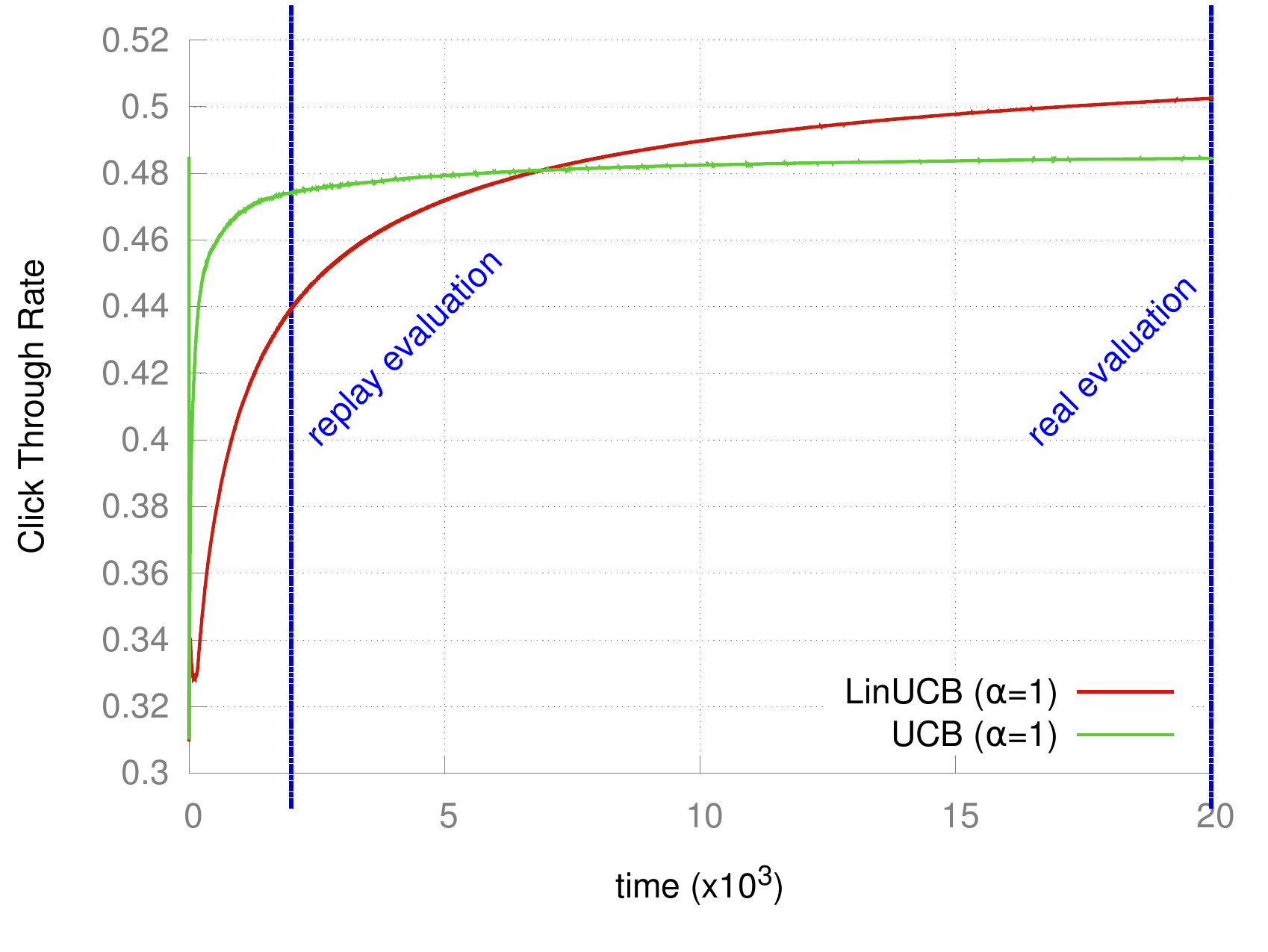}
  \caption{CTR over time of LinUCB (good model, slow learning) and UCB
    (too simple model but fast learning) when played $T=20,000$
    times on $\cal D$ (see section \ref{sec:artif} for the
    model). Replay emulates only $\frac{T}{K}$ steps (see blue line)
    which is clearly misleading if we are interested in performance
    over $T$ steps.}
\label{fig:biased_replay}
\end{figure}

\section{Bootstrapped Replay on Expanded Data}

Now that the shortcoming of the replay method has been understood, we
look for an other offline evaluation protocol that does not suffer
from the time acceleration issue. The idea we propose stems from the
idea of bootstrapping, introduced by \cite{citeulike:2825416}. Thus
let us remind the standard bootstrap approach. Basically, the idea is to compute
the empirical distribution $\hat{\D}_T$ 
 of an estimator $\hat{\theta}_T$ computed on $T$
observations. To do so, one only has access to a dataset $S$ of size
$T$. Therefore $B$ new datasets $S_1,... S_B$ of size $T$ are generated by making
bootstrap resamples from $S$.
A bootstrap resample is generated via drawing $T$ samples with
replacement. Note that this bootstrap procedure is a way to approximate $\D$,
the underlying distribution of the data. Computing $\hat{\theta}_T$ on all the $S_i$ yields 
$\widehat{CTR}(T)$, an estimation of $CTR(T)$. From a theoretical point of view and under mild
assumptions,  $\xi\left( \widehat{CTR}(T) \right)$ converges with no bias 
 at a speed in $O(1/T)$.
This means that under a assumption of the concentration speed of a statistic we are able to estimate the confidence interval of the mean of the statistic much faster than its mean.
Recall that $\xi$ can be any measure of accuracy (defined in terms of bias, variance, confidence intervals, \ldots) over the statistic we want to study. Here we are interested in confidence intervals over $CTR_A(T)$.
\begin{algorithm}[h!]
  \caption{\textit{Bootstrapped Replay on Expanded Data (BRED)}.
    \newline
    We sketch this algorithm so that it looks very much like
    \textit{replay} in \cite{LiCLW11}. Thus an algorithm is a
    function that maps an history of events $h^{(b)}$ to a policy
    $\pi$ which itself maps contexts to actions. This makes the learning
    process of the algorithm appear clearly. A computationally efficient
    implementation would be slightly different. Notice also that for clarity,
    we only compute $\hat{g}_\algo(T)$ and omit ${\widehat{CTR}}_\algo(T)$.
}
  \underline{Input}
  \begin{itemize}
    \item A (contextual) bandit algorithm $\algo$
    \item A dataset $S$ of $T$ triplets $(x, a, r)$
    \item An integer $B$
  \end{itemize}
  \underline{Output}: An estimate of $g_\algo$\\
  \begin{algorithmic}
    \STATE $h^{(b)} \leftarrow \emptyset$, $\forall b \in \left\{1..B\right\}$  \textit{ /*empty history*/}
    \STATE $\widehat{G}_{\algo}^{(b)} \leftarrow 0$, $\forall b \in \left\{1..B\right\}$ 
    \STATE $T^{(b)} \leftarrow 0$, $\forall b \in \left\{1..B\right\}$
    \STATE \textit{/* Bootstrap loop*/} 
    \FOR{$b \in \left\{1..B\right\}$}
      \STATE \textit{/* estimation of $CTR_\algo^{(b)}(T)$*/}
      \FOR{$i \in \left\{1..T\times{}K\right\}$ }
        \STATE Sample with replacement $(x,a,r)$ of $S$
        \STATE $x \leftarrow$ JITTER$(x)$ \textit{ /*optional*/}
        \STATE $\pi \gets \algo (h^{(b)})$
        \IF{$\pi(x) = a$}
          \STATE add $(x,a,r)$ to $h^{(b)}$
          \STATE $\widehat{G}_\algo^{(b)} \leftarrow \widehat{G}_\algo^{(b)} + r$
          \STATE $T^{(b)} \leftarrow T^{(b)}+1$
      \ENDIF
    \ENDFOR
  \ENDFOR
  \STATE \textbf{return} $\frac{1}{B} \sum\limits_{b=1}^B \frac{\widehat{G}_{\algo}^{(b)}}{T^{(b)}}$
  \end{algorithmic}
  \label{bred}
\end{algorithm}
The core idea of the evaluation protocol we propose in this paper is
inspired by bootstrapping and inherits its theoretical
properties. Using our notations, here is the description of this new
method. From a dataset of size $T$ with $K$ possible choices of action
at each step - we do not require $K$ to be constant over time -,
we generate $B$ datasets of size $K.T$ by sampling with
replacement, following the non-parametric bootstrap procedure. Then
for each dataset $S_b$ we
use the classical replay method to compute an estimate
$\hat{g_\algo}^{(b)}(T)$.
Therefore $\algo$ is evaluated on $T$ records
on average. This step can be seen as a subsampling step that allows to return in
the classical bootstrap setting. Thus note that it would not work for a purely
deterministic policy, that for obvious reason would not take advantage
of the data expansion (an assumption in the formal analysis will reflect
this fact). $\hat{g}_\algo(T)$ is given by averaging the
$\hat{g_\algo}^{(b)}(T)$. Together, the $\hat{g_\algo}^{(b)}(T)$ are
also an estimation of
$CTR_\algo(T)$, the distribution of the CTR of $\algo$ after $T$ interactions
with $\D$ on which we can compute our estimator quality assessment
$\xi$. More formally, the bootstrap estimator of the density of
$CTR_\algo(T)$ is estimated as follows:
$$
{\widehat {CTR}_\algo(T)}(x) = \frac1B \sum_{b=1}^{B} I\left(\sqrt{T} \left[
    \frac{\hat{g}_\algo^{(b)}(T) - \hat{g}_\algo(T)}{\hat{\sigma}_\algo(T)}
    \right] \leq  x\right)
$$ 
where $\hat{\sigma}_\algo(T)$ is the empirical standard deviation
obtained when computing the bootstrap estimates $\hat{g_\algo}^{(b)}(T)$. 
The complete procedure, called Bootstrapped Replay on Expanded
Data (BRED), is implemented in algorithm \ref{bred}.

To complete the BRED procedure, one last detail is necessary. Each
record of the original dataset $S$ is contained $K$ times in
expectation in each expanded dataset $S_b$. Therefore a learning
algorithm may tend to overfit which would bias the estimator. To
prevent this from happening, we introduce a small amount of Gaussian
noise on the contexts. This technique is known as jittering and is
well studied in the neural network field
\cite{citeulike:2512648}. The goal is the same, that is avoiding
overfitting. In practice however it is slightly different as neural network are generally not learning online but on batches of data, each data being used several times during learning. In bootstrapping theory this technique is known as the smoothed
bootstrap and was introduced by \cite{silverman1987bootstrap}. We
mentioned that the bootstrap 
resampling is a way to approximate
$\D$. The smoothed bootstrap goes further by performing a kernel
density estimation (KDE) of the data and sampling from it. Sampling from a
KDE of the data where the kernel is Gaussian of bandwidth $h$ is
equivalent to sampling a record uniformly from $S$ and applying a random noise
sampled from ${\cal N}(0, h^2)$, which is what jittering does. The usual purpose
of doing so in statistics is to get a less noisy empirical distribution for the
estimator. Note that here we perform a partially smoothed bootstrap as
we only apply a KDE on the part of $\D$ that generates the contexts.

\section{Theoretical analysis}

In this section, we make a theoretical analysis of our evaluation
method BRED. The core loop in BRED is a bootstrap loop; henceforth, to
complete this analysis, we first restate the theorem \ref{theo:1}
which is a standard result of the bootstrap asymptotic analysis
\cite{blb}. Notice a small detail: each bootstrap step estimates a
realization of $CTR_\algo(T)$. The number of evaluations - which is also the number of non rejects - is a random
variable denoted $T^{(b)}$.

\begin{theorem}\label{theo:1}
  Suppose that:
  \begin{itemize}
    \item $\algo$ is a recommendation algorithm which generates a fixed
      policy over time (this hypothesis can be weakened as
      discussed in remark 2),
    \item $K$ items may be recommended at each time step,
    \item $\xi\left(CTR_\algo(T)\right)$ admits an expansion as an asymptotic series 
      $$
        \xi\left(CTR_\algo(T)\right) = z + \frac{p_1}{\sqrt{T}}+\ldots+ \frac{p_\alpha}{T^{\alpha/2}} +o\left(\frac1{T^{\alpha/2}}  \right)
      $$
      where $z$ is a constant independent of the distribution $\D$
      (as defined in Sec.\@ \ref{cbandits}), and the $p_i$ are
      polynomials in the moments
      of $CTR_\algo(T)$ under $\D$ (this hypothesis is
      discussed and explained in remark 1),
    \item The empirical estimator $\xi\left( {\widehat
       CTR}_\algo(T)\right)$ admits a similar expansion: 
      \begin{equation} 
	\begin{split}
	  z + \frac{\widehat p_1}{\sqrt{T}}+\ldots+ \frac{\widehat p_\alpha}{T^{\alpha/2}} +O\left(\frac1{T^{\alpha/2}}  \right).
	\end{split}
      \end{equation}
  \end{itemize}
  Then, for $T\leq T^{(b)} \times{} K$ and assuming finite first and second moments of $\widehat{CTR}_\algo(T)$, with high probability:
  \begin{equation}
    \begin{split}
      \left|   \xi\left( CTR_\algo(T) \right) - \xi\left( {\widehat CTR}_\algo(T)\right)   \right| = \\
      O\left(\frac{\mbox{Var}( \widehat p_\alpha^{(1)} - p_\alpha|D_T)}{\sqrt{T\cdot B}}\right) +
      O\left(\frac1{T}\right) + O\left(\frac1{T \sqrt{T}}\right)
    \end{split}
    \label{boosteqn}
  \end{equation}
  where $\D_T$ is the resampled distribution of $\D$ using $T$ realizations. 
  \label{boot}
\end{theorem} 
\begin{proof}
  it is actually a straightforward adaptation of the proof of theorem
  3 of \cite{blb}. 
	 Also note
  that this theorem is a reformulation of the bootstrap main
  convergence result as introduced by \cite{citeulike:2825416}.
\end{proof}

Now, we use theorem \ref{theo:1} to bound the error made by BRED in
the theorem \ref{theo:2}.

\begin{theorem}\label{theo:2}
  Assuming that
  $$
    \xi\left( CTR_\algo(T) \right)= 
	z + \frac{p_1}{\sqrt{T}}+\ldots+ \frac{p_\alpha}{T^{\alpha/2}} +o\left(\frac1{T^{\alpha/2}}  \right)
  $$

  Then for algorithm $\algo$ producing a fixed policy over time,
  BRED applied on a dataset of size $T$ evaluates the expectation of
  the $CTR_\algo$ with no bias and with high probability for $B$
  and $T$ large enough:

  $$
  \left|   \xi\left( CTR_\algo(T) \right) - \xi\left( {\widehat CTR}_\algo(T)\right)   \right| = O\left(\frac{1}{T}\right)
  $$
  This means that the convergence of the estimator of $\xi\left(
CTR_\algo(T) \right)$ is much faster than the convergence of the
estimator of $g_\algo(T)$ (which is in $O(1/\sqrt{T})$. This will
allow a nice control of the risk that $\hat{g}_\algo(T)$ may be badly
evaluated.
\label{BREDproof}
\end{theorem}

The sketch of the proof of theorem \ref{theo:2} is the following:
first we prove that the replay strategy is able to estimate the
moments of the distribution of $CTR_\algo$ fast enough with respect to
$T$. The second step consists in using classical
results from bootstrap theory to guarantee the unbiased convergence of
the aggregation $\widehat{CTR}_A(T)$ to the true distribution with
an $O(\frac1{T})$ speed. The rational behind this is that the
gap introduced by the subsampling will be of the order of
$O(\frac1{\sqrt{T B}})$.

\begin{proof}

At each iteration of the bootstrap loop (indexed by $b$), BRED is
estimating the CTR using the replay method on a dataset of size $T' =
K \times{} T$. As the actions in $S$ were chosen uniformly at random, we have $\E(T^{(b)}) = T' / K = T$.

As the policy is fixed, we can use the multiplicative Chernoff's bound
as in \cite{LiCLW11} to obtain for all bootstrap step $b$:

$$
\proba{} \left( \left| T^{(b)} - T \right| \geq {\gamma_1 T}\right)
\leq \exp\left( - \frac{T \gamma_1^2}{3} \right) 
$$

for any $\gamma_1>0$ (where $\proba(e)$ denotes the probability of
event $e$). A similar inequality can be obtained with $\E(\widehat
G_A)=T g_A$:

$$
\proba{} \left( \left| \widehat G_A - { T g_A} \right| \geq {\gamma_2 T g_A}\right) \leq \exp\left( - \frac{T g_A\gamma_2^2}{3} \right) 
$$

Thus with $\gamma_1 = \sqrt{\frac{3}{T} \ln \frac4\delta}$ and $\gamma_2 = \sqrt{\frac{3}{T g_A} \ln \frac4\delta} $ 
using a union bound over probabilities, we have with probability at least $1-\delta$: 
$$
{1-\gamma_1} \leq \frac {T^{(b)}}{T} \leq  {1+\gamma_1} 
$$
$$
g_A {1-\gamma_2} \leq \frac {\widehat G_A}{T} \leq g_A  {1+\gamma_2} 
$$
which implies 
$$
\left| \frac{\widehat G_A}{T^{(b)}} - g_A \right| \leq \frac{(\gamma_1 + \gamma_2)g_A}{1-\gamma_1} = O\left( \sqrt{\frac{ g_a}{T}} \ln \frac1\delta\right)
$$
So with high probability the first moment of $\widehat{CTR}_A(T^{(b)})$ as estimated by the replay method admits an asymptotic expansion in
 $1/\sqrt{\E(T^{(b)})} = 1/T$.

Now we need to focus on higher order terms. All the moments are finite because the reward distribution over $\knownRatings$ is bounded. 
Recall that by hypothesis $\xi \left( CTR_\algo(T) \right) $ admits a $\alpha^{{th}}$ order term:
$$ 
p_\alpha = \E_D \left( CTR_\algo(T^{(b)})^\alpha \right)
$$ 

The Chernoff's bound can be applied to $|(T^{(b)})^\alpha-T^\alpha|$ and $|\widehat G_A^\alpha - T^\alpha g_A^\alpha|$ leading  to 

$$
\left| \frac{\widehat G_A^\alpha}{(T^{(b)})^\alpha} - g_A^\alpha \right| = O\left( {\left(\frac{ g_a}{T} \right)}^{\frac{\alpha}2} \ln \frac1\delta\right)
$$

With probability at least $1-\delta$. So for a large enough $T^{(b)}$,
$\xi\left( \widehat{CTR}_A(T^{(b)}) \right)$ admits a expansion in
polynomials of $1/{\sqrt{T}}$.  Thus theorem 1 applies and the
aggregation of all the $\hat{g}^{(b)}_\algo(T^{(b)})$ allows an
estimation of $CTR_A(T)$. For a large enough number of bootstrap
iterations (the value of $B$ in BRED), we obtain a convergence speed
in $O(1/T)$ with high probability, which concludes the proof.
\end{proof}

After this analysis, we make two remarks about the assumptions that
were needed to establish the theorems.

\textit{Remark 1}: The key point of the theorems is the existence of an
asymptotic expansion of $CTR_\algo(T)$ and ${\widehat {CTR}}_\algo(T)$
in polynomials of $1/\sqrt{T}$. This is a natural hypothesis for $CTR_\algo(T)$ because the CTR
is an average of bounded variables (probabilities of click). Note that
the proof of theorem 2 shows that although $T^{(b)}$ is random the
expansion remains valid anyway. For a
contextual bandit algorithm $\algo$ producing a fixed policy, the mean is
going to concentrate according to the central limit theorem
(CLT). Furthermore this hypothesis, omnipresent in bootstrap theory
\cite{citeulike:2825416}, is for instance justified in econometrics by the fact that all
the common estimators respect it \cite{horowitz2001bootstrap}. Yet
this assumption is not verified for a static deterministic policy.

\begin{figure}[b!]
  \includegraphics[width=0.47\textwidth]{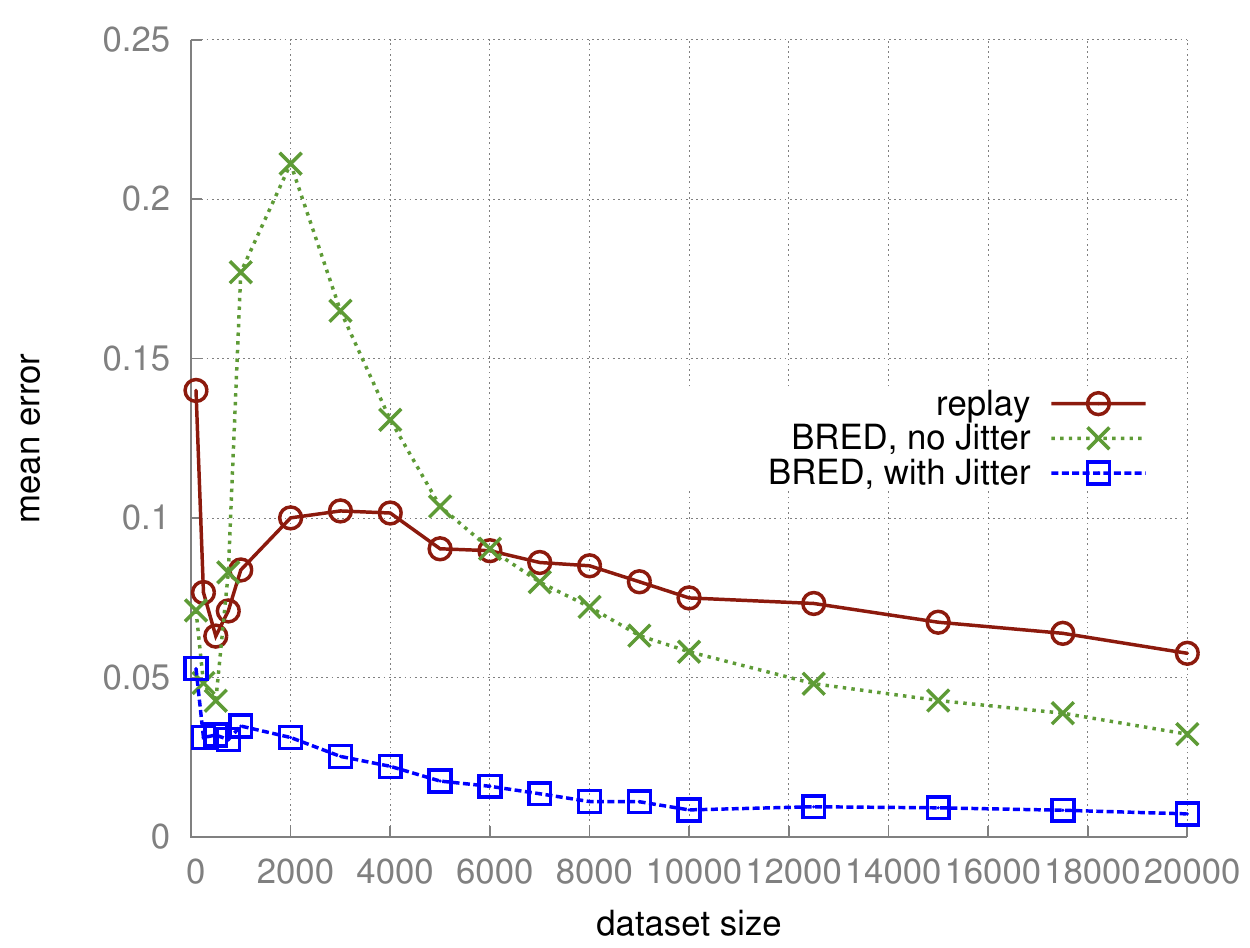}
  \caption{Mean absolute error of the CTR of LinUCB estimated by different
    methodologies. Conducted on artificial datasets as described by
    section \ref{sec:artif}. The lower, the better. Jittering
    ($h=\frac{50}{\sqrt{T}}$ here) is
    actually efficient to avoid overfitting issues. The rather small
    error rates for very small datasets are due to the fact that on
    too small datasets all the recommender algorithm tend to make
    random choices which are not very hard to evaluate.
  }
  \label{artif-LinUCB-fig}
\end{figure}

\textit{Remark 2} Let us consider algorithms that produce a policy
which changes over time (a learning algorithm in particular). After a
sufficient amount of recommendations, a reasonable enough algorithm
will produce a policy that will not change any longer (if the world is
stationary).  Thus again, the CLT will apply and we will observe a
convergence of $\hat{g}_\algo(T)$ to its limit in
$1/\sqrt{T}$. Nevertheless nothing holds true here when the algorithm
is actually learning. This is due to the fact that the Chernoff bound
no longer applies as the steps are not independent. However the
behavior of classical learning algorithms are smooth, especially when
randomized (see \cite{exp3} for an example of a randomized version of
UCB). \cite{LiCLW11} argue that in this case convergence bounds may be derived for replay (which then could be applied to BRED) at the cost of a
much more complicated analysis including smoothness assumptions. For
non reasonable algorithms and thus in the general case, no guarantees
can be provided. By the way note that a very intuitive way to justify
Jittering is to consider that it helps the Chernoff bound being "more
true" in the case of a learning algorithm.


\section{Experiments in realistic settings}
\label{expes}
As we proved that BRED has promising theoretical guarantees in the setting introduced in \cite{LiCLW11}, let us now compare its empirical performance to that of the replay method 

\subsection{Synthetic data and discussing Jittering}
\label{sec:artif}

The first set of experiments was run on synthetic data. Indeed, we needed to be able to compare the errors of estimation of the two methods on various fixed size datasets relatively to the ground truth: an evaluation against the model itself.
Before going any further, let us describe the model we used. It is a linear model with Gaussian noise (as in \cite{LinUCB}) and was built as follows:
\begin{itemize}
  \item a fixed action set (or news set) of size $K=10$.
  \item The context space is $\mathcal{X} = \R^{15}$. Each context $x$ is generated as a sum $c+n$ where $c \sim
\mathcal{N}(0,1)$ and $n \sim \mathcal{N}(0,\frac{1}{2})$.
  \item The CTR of a news $i$ displayed in a context $x$ is given
linearly by $q_i+w_i^Tc$. Note the non-contextual element $q_i$ and
that the noise $n$ is ignored.
  \item Finally there are two kinds of news: (i) 4 ``universal'' news
that are interesting in general like \textit{Obama is re-elected} and
for which $q_i \sim \mathcal{U}(0.4,0.5)$ is high and
$w_i=0_{15}$. (ii) 6 specific news like \textit{New Linux distribution
released} for which $q_i \sim \mathcal{U}(0.1,0.2)$ is low and $w_i$
consists of zeros except for a number $m$ of relevant weights sampled
from $\mathcal{N}(0,\frac{1}{5})$.
\end{itemize}
 
A non contextual approach would perform decently by quickly learning
the $q_i$ values. Yet LinUCB \cite{LinUCB}, a contextual bandit
algorithm will do better by learning when to recommend the specific
news. Figure \ref{artif-LinUCB-fig} displays the
results and interpretation of an experiment which consists in
evaluating LinUCB($\alpha=1$) using the different methods. It is
clear that BRED converges much faster than the replay method.

\emph{Remark}: As it can be seen on Figure \ref{artif-LinUCB-fig},
jittering is very important to obtain good performance when evaluating
a learning algorithm. Empirically, a good choice for the level of
jitter seems to be a function in $O\left(\frac{1}{\sqrt{T}}\right)$,
with $T$ the size of the dataset. Note that this is proportional to
the standard distribution of the posterior of the data. The results
confirm our intuition: jittering is very
important when the dataset is small but gets less and less necessary
as the dataset grows.


\subsection{Real data}
\label{sec:real}

Adapting replay to a real life dataset, corresponding to dynamic
recommendation is straightforward although it leads to biased
estimations. BRED really needs the assumption of a static world in
order to perform the bootstrap resamples. Therefore BRED needs to be
run on successive windows of the dataset on which a static assumption
can be made. This creates a bias/variance trade-off: if the windows
are too big, some dynamic properties of the world may be erased (bias). On
the contrary, too smalls window will lead to very variate bootstrap
estimates.
To simplify things, we ran experiments assuming a
static world on small portions of the Yahoo! R6B dataset. We actually
took the smallest number of portions such that a given portion has a
fixed pool of news ($\approx 630$ portions). This experiment is
similar to what is done in \cite{LiCLW11}: the authors measured the
error of the estimated CTR of UCB ($\alpha=1$) by the replay method on
datasets of various sizes relatively to what they call the ground
truth: an evaluation of the same algorithm on a real fraction of the
audience.  As we obviously cannot do that, we used a simple trick.
For each portion $i$ of size $T_i$ with $K_i$ news, we computed an
estimation of the ground truth $g_\algo\left(\frac{T_i}{K_i}\right)$
by averaging the estimated CTR of UCB using the replay method on 100
random permutations of the data.  For each portion the experiment
consists in subsampling $T_i/K_i$ records and evaluating UCB using
replay and BRED on this smaller dataset to estimate the ground truth
using less data, faking time acceleration.  The results and
interpretation are shown on Figure \ref{real-data-fig}: the better accuracy of BRED is very clearly illustrated.

\begin{figure}[htbp]
  \vspace{-1cm}
  \includegraphics[width=0.48\textwidth]{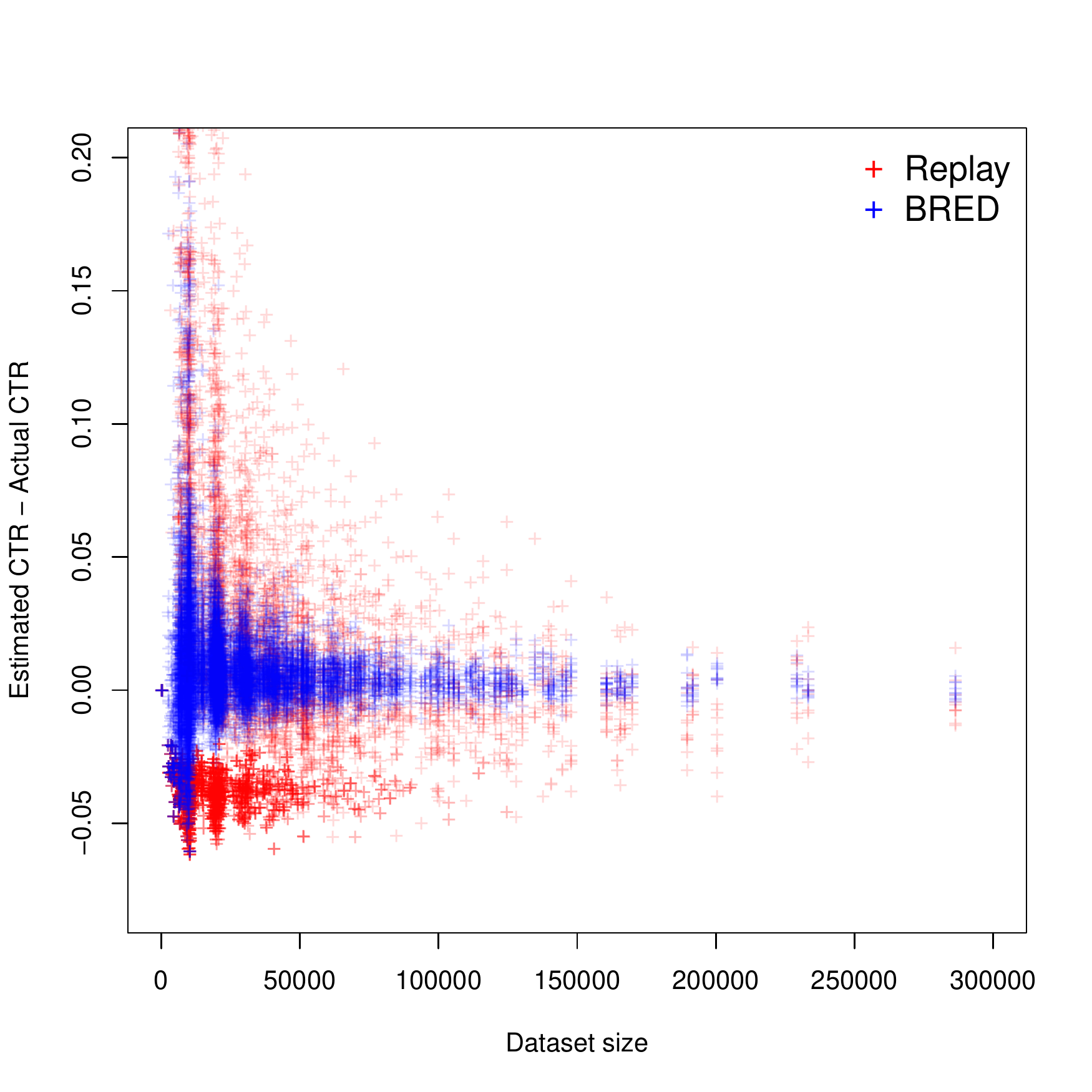}
  \caption{The difference between the estimated CTR and the actual one
    on some batches extracted from the Yahoo! R6B dataset for a
    UCB. Batches are build as explained in section \ref{sec:real}. The
    closer to 0, the better. Please note that the replay method tends
    to under-estimate the true CTR for small batches. This is due to
    the fact that UCB does not have enough time to reach its
    actual CTR.
  }
  \label{real-data-fig}
\vspace{-0.4cm}
\end{figure}

\section{Conclusion}

In this paper, we studied the problem of recommendation system
evaluation, sticking to a realistic setting: we focused on obtaining a
methodology for practical offline evaluation, providing a good
estimate using a reasonable amount of data. Previous methods are
proved to be asymptotically unbiased with a low speed of convergence
on a static dataset, but yield counter-intuitive estimates of
performance on real datasets. Here, we introduce BRED, a method with a
much faster speed of convergence on static datasets (at the cost of
loosing unbiasedness) which allows it to be much more accurate on
dynamic data. Experiments demonstrated our point; they were performed
on a publicly available dataset made from Yahoo!\@ server logs and on
synthetic data presenting the time acceleration issue. This paper was
also meant to highlight the time acceleration issue and the misleading
results given by a careless evaluation of an algorithm. Finally our
method comes with a very desirable property in a context of minimizing
the risks entailed by putting online a new RS: an extremely accurate
estimation of the variability of the estimator it provides.
 

An interesting line of future work is the automatic selection of the
Jittering bandwidth. Note that this problem is extensively studied in
the context of KDE \cite{scott1992multivariate}. 

A possible extension of this work is to use BRED to build a "safe
controller". Indeed, when a company uses a recommendation system that
behaves according to a certain policy $\pi$ that reaches a certain
level of performance, the hope is that when changing the
recommendation algorithm, the performance will not drop. As an
extension of the work presented here, it is possible to collect some
data using the current policy $\pi$, compute small variations of $\pi$
with tight confidence intervals over their CTR and then replace the
current policy $\pi$ with the improved one. This may be seen as a kind
of ``gradient'' ascent of the CTR in the space of policies.

\bibliographystyle{icml2014}
\bibliography{biblio}

\newpage
\appendix

\begin{center}
  \begin{Large}
    Supplementary material
  \end{Large}
\end{center}

The detailed implementation of \textit{replay} using our notations is
given in algorithm \ref{replay}. Note that apart from notations, no
modification are made with regards to the original presentation in
\cite{Langford_ExploScav_08, LiCLW11}. Figure \ref{artif-UCB-fig} is
the same experiment as in section \ref{sec:artif} but with a
non-contextual algorithm UCB: this plot exhibits both the importance
of Jittering and the improvement brought by our method compared to the
state of the art.

\begin{figure}[htbp]
  \includegraphics[width=0.47\textwidth]{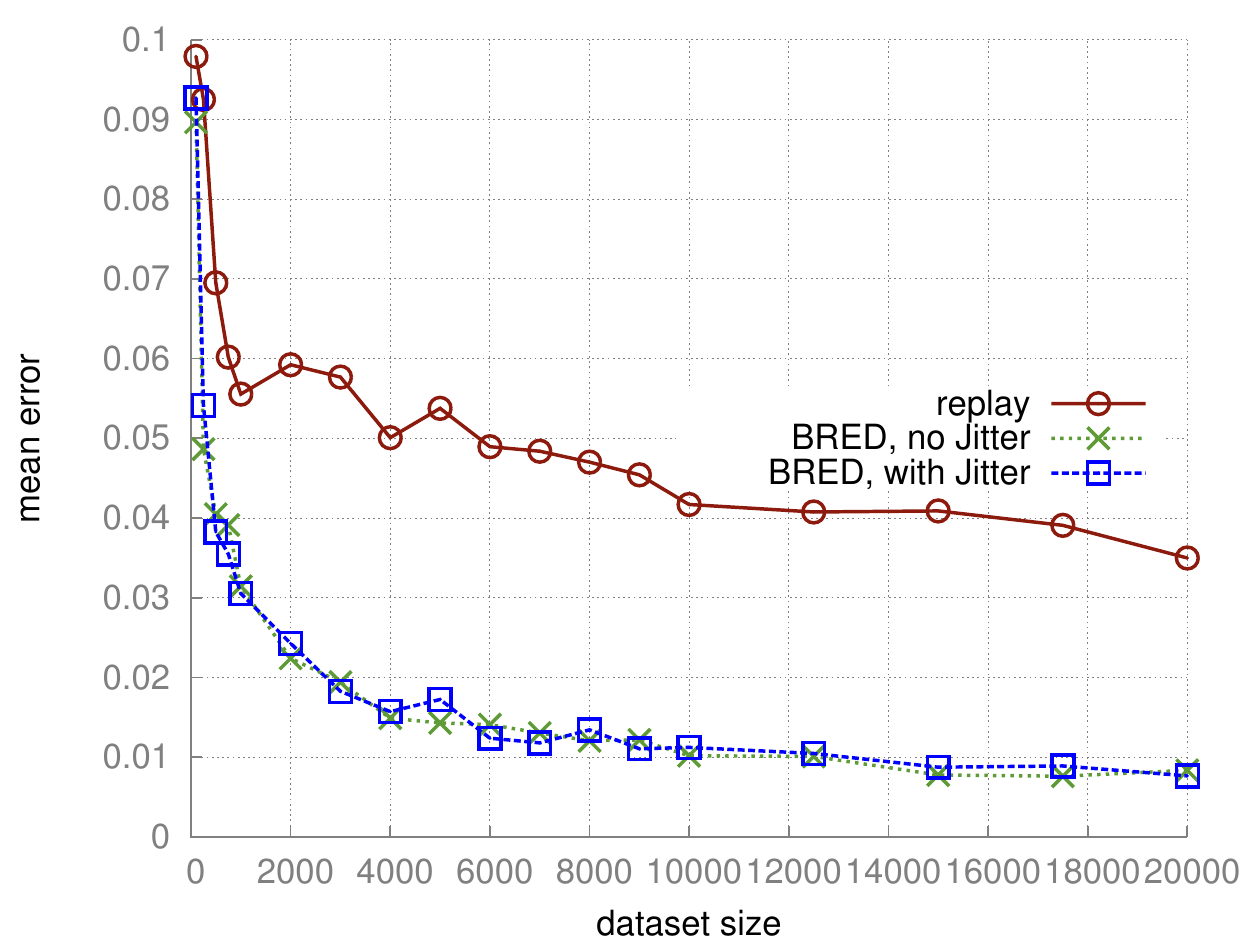}
  \caption{
  Mean of the absolute value of the difference between the true CTR of a UCB and the estimated one for different methodologies. Conducted on artificial dataset as described in the section 6.1 of the main paper. The lower, the better. Jittering is useless here because UCB does not use the context.}
  \label{artif-UCB-fig}
\end{figure}

\begin{algorithm}
  \caption{\textit{Replay method} \cite{Langford_ExploScav_08, LiCLW11}.
    \newline
    Remark: for the sake of the precision of the specification of the algorithm, we use a history $h$ which is the list of triplets $(x,a,r)$ that have yet been used to estimate the performance of the algorithm $\algo$. The goal is to avoid hiding internal information maintenance in $\algo$; a real implementation may be significantly different for the sake of efficiency, by learning incrementally.}
  \underline{Input}: 
  \begin{itemize}
    \item A contextual bandit algorithm $\algo$
    \item A set $S$ of $L$ triplets $(x, a, r)$
  \end{itemize}
  \underline{Output}: An estimate of $g_\algo$
  \begin{algorithmic}
    \STATE $h \leftarrow \emptyset$
    \STATE $\widehat{G}_{\algo} \leftarrow 0$
    \STATE $T \leftarrow 0$
    \FOR{$t \in \left\{1..L\right\}$}
      \STATE Get the $t$-th element $(x,a,r)$ of $S$
      \STATE $\pi \gets A (h)$
      \IF{$\pi(x) = a$}
        \STATE add $(x,a,r)$ to $h$
        \STATE $\widehat{G}_\algo \leftarrow \widehat{G}_\algo + r$
        \STATE $T \leftarrow T+1$
      \ENDIF
    \ENDFOR
    \STATE \textbf{return} $\frac{\widehat{G}_{\algo}}{T}$
  \end{algorithmic}
  \label{replay}
\end{algorithm}

\end{document}

%% file: raccourcis.tex
\renewcommand{\epsilon}{\varepsilon}

\newcommand{\E}{\mathbb{E}}

\newcommand{\R}{\mathbb{R}}

\def\beglab{\begin{equation} \label}
\def\endlab{\end{equation}}
\def\beglabc{\begin{equation*} }
\def\endlabc{\end{equation*}}